\documentclass[runningheads]{llncs}
\usepackage{graphicx}
\usepackage{amsmath,amssymb,amsfonts}
\usepackage{cite}
\usepackage{csquotes}
\usepackage[T1]{fontenc}

\begin{document}
\title{Homomorphisms Between Transfer, Multi-Task, and Meta-Learning Systems}
\author{Tyler Cody}
\authorrunning{T. Cody}
\institute{National Security Institute \\ Virginia Tech \\ Arlington, VA, USA \\
\email{tcody@vt.edu}}
\maketitle              % typeset the header of the contribution
\begin{abstract}

Transfer learning, multi-task learning, and meta-learning are well-studied topics concerned with the generalization of knowledge across learning tasks and are closely related to general intelligence. But, the formal, general systems differences between them are underexplored in the literature. This lack of systems-level formalism leads to difficulties in coordinating related, inter-disciplinary engineering efforts. This manuscript formalizes transfer learning, multi-task learning, and meta-learning as abstract learning systems, consistent with the formal-minimalist abstract systems theory of Mesarovic and Takahara. Moreover, it uses the presented formalism to relate the three concepts of learning in terms of composition, hierarchy, and structural homomorphism. Findings are readily depicted in terms of input-output systems, highlighting the ease of delineating formal, general systems differences between transfer, multi-task, and meta-learning.

\end{abstract}
%
%
%

%\section{Notation}

%The Cartesian product is denoted $\times$. Given system $S \subset \times \{V_i\}$ for i = 0, ..., I, $\overline{S}$ denotes the component sets of $S$, i.e., $\overline{S} = \{V_0, ..., V_I\}$.

\section{Introduction}

Transfer learning, multi-task learning, and meta-learning are three different concepts of learning that aim to generalize knowledge across learning tasks. As such, they are common topics in artificial general intelligence \cite{al2018continuous, sheikhlar2020autonomous}. They are informally described as similar in their respective, prominent surveys \cite{pan2009survey, zhang2021survey, vanschoren2018meta}. Formally, however, the general systems character of this similarity is left undiscussed. Likely, this is because the formalism of their respective learning algorithms quickly represents their differences. While this gap may seem inconsequential to algorithm designers, who typically work very closely to solution methods, to systems engineers, this gap muddles basic questions about composition and hierarchy.   

In this manuscript, a recently proposed abstract systems theory (AST) model of learning \cite{cody2021mesarovician, cody2021systems} is used to formally relate transfer, multi-task, and meta-learning. Each concept of learning is modeled as an abstract system \cite{mesarovic1989abstract}, i.e., as a relation on component sets, and their structural homomorphism is studied. 
%Formal-minimalist AST has the aim of formalizing block-diagrams without loss of generality \cite{dori2019system}. This desiderata enforces a systems-level focus, leaving detailed modeling to solution methods. Herein, it means results are easily depicted. 
The presented results extend previous work that synthesizes AST with statistical learning theory \cite{cody2021mesarovician} and transfer learning \cite{cody2019systems, cody2021systems} with novel definitions of multi-task and meta-learning as abstract systems, and with an investigation of their structural similarities.

This manuscript is structured as follows. First, preliminaries on abstract learning systems and transfer learning systems are given in Sections \ref{sec:als} and \ref{sec:tls}. Subsequently, multi-task learning and meta-learning are formalized as systems from their informal descriptions in Section \ref{sec:mtlml}, and the homomorphism between transfer, multi-task, and meta-learning is investigated in Section \ref{sec:comp}. The manuscript concludes with a synopsis and remarks on the pitfalls of the existing informal taxonomy in light of the presented material.

%To the extent artificial general intelligence (AGI) is concerned with generalizing knowledge across learning tasks, related first-principles engineering practice ought to be equipped with appropriate models for related learning phenomena. 

%Engineers need to be able compare various mechanisms for generalizing knowledge without relying on the detailed mathematics of solution methods. This is especially true in long-lived systems that need to account for solution methods being developed during their life cycle.

\section{Abstract Learning Systems}
\label{sec:als}

Abstract systems $S$ are relations on (non-empty) abstract sets $$S \subset \times \{V_i|i=1, ..., I\},$$ where $\times$ is the Cartesian product, $V_i$ are (component) sets, and $\overline{S} = \{V_i|i=1,...,I\}$ \cite{mesarovic1989abstract}. Input-output systems are (elementary) systems $$S \subset \times \{\mathcal{X}, \mathcal{Y}\},$$ where $\mathcal{X} \cap \mathcal{Y} = \varnothing$, $\mathcal{X} \cup \mathcal{Y} = \overline{S}$, and $\varnothing$ is the empty set. The set $\mathcal{X}$ is termed the input and the set $\mathcal{Y}$ is termed the output. Functional systems are input-output systems of the form $S: \mathcal{X} \to \mathcal{Y}$. AST is primarily concerned with input-output systems, with their composition, and with categories of systems \cite{mesarovic1989abstract}.

Recent work presented a stratified model of abstract learning systems as a cascade connection of learning algorithms $A:D \to \Theta$ and hypotheses $H: \Theta \times \mathcal{X} \to \mathcal{Y}$ where $D$ are data and $\Theta$ are parameters \cite{cody2021mesarovician}. This follows the treatment of learning as function approximation \cite{wang2016different}. Learning systems are defined as follows.
\begin{definition}[Learning Systems.] \\
    A learning system $S$ is a relation
    $$S \subset \times \{A, D, \Theta, H, \mathcal{X}, \mathcal{Y} \}$$
    such that
    \begin{gather*}
        D \subset \mathcal{X} \times \mathcal{Y}, A:D \to \Theta, H:\Theta \times \mathcal{X} \to \mathcal{Y} \\
        (d, x, y) \in \mathcal{P}(S) \leftrightarrow (\exists \theta) [(\theta, x, y) \in H \wedge (d, \theta) \in A]
    \end{gather*}
    where 
    $$x \in \mathcal{X}, y \in \mathcal{Y}, d \in D, \theta \in \Theta.$$ 
    The algorithm $A$, data $D$, parameters $\Theta$, hypotheses $H$, input $\mathcal{X}$, and output $\mathcal{Y}$ are the component sets of $S$, $\mathcal{P}$ is the power set, and learning is specified in the relation among them.
    \label{def:ls}
\end{definition}
This AST model of learning is depicted in Figure \ref{fig:learning-systems} at the elementary (input-output) and cascade levels of abstraction (as presented in \cite{cody2021mesarovician}).

\begin{figure*}[t]
    \centering
    \includegraphics[width=.45\textwidth]{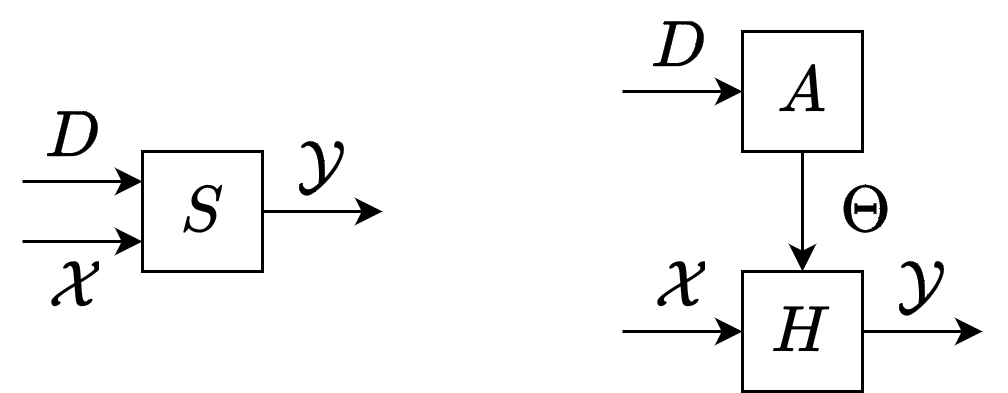}
    \caption{Learning systems at the elementary (left) and cascade (right) levels of abstraction \cite{cody2021mesarovician}.}
    \label{fig:learning-systems}
\end{figure*}

\section{Transfer Learning Systems}
\label{sec:tls}

The concept of learning tasks is widely used in artificial intelligence \cite{thorisson2016artificial}. Transfer learning is conventionally defined in terms of domains $\mathcal{D} = \{\mathcal{X}, P(X)\}$ and tasks $\mathcal{T} = \{\mathcal{Y}, P(Y|X)\}$, where $P$ denotes a probability measure. Given a source domain $\mathcal{D}_S$ and learning task $\mathcal{T}_S$, a target domain $\mathcal{D}_T$ and learning task $\mathcal{T}_T$ , Pan and Yang define transfer learning as a learning paradigm that \cite{pan2009survey},
\begin{quote}
    ``aims to help improve the learning of the target predictive function $f_T$\footnote{$f_T \sim P(Y_T|X_T)$} in $\mathcal{D}_T$ using the knowledge in $\mathcal{D}_S$ and $\mathcal{D}_T$, where $\mathcal{D}_S \neq \mathcal{D}_T$ or $\mathcal{T}_S \neq \mathcal{T}_T$.''
\end{quote}
Alternatively, previous work describes transfer learning as \cite{cody2021systems}, 
\begin{quote}
``...a relation on the source and target (learning) systems that combines knowledge from the source with data from the target and uses the result to select a hypothesis that estimates the target learning task.''
\end{quote}
Transfer learning systems are defined as follows.
\begin{definition}[Transfer Learning System.] \\
    Given source and target learning systems $S_S$ and $S_T$
    \begin{gather*}
        S_S \subset \times \{ A_S, D_S, \Theta_S, H_S, X_S, Y_S \} \\
        S_T \subset \times \{ A_T, D_T, \Theta_T, H_T, X_T, Y_T \} 
    \end{gather*}
    a transfer learning system $S_{Tr}$ is a relation on the component sets of the source and target systems $S_{Tr} \subset \overline{S_S} \times \overline{S_T}$ such that
    $$K_S \subset D_S \times \Theta_S, D \subset D_T \times K_S$$
    and
    \begin{gather*}
        A_{Tr}: D \to \Theta_{Tr}, H_{Tr}: \Theta_{Tr} \times X_T \to Y_T \\
        (d, x_T, y_T) \in \mathcal{P}(S_{Tr}) \leftrightarrow \\
        (\exists \theta_{Tr})[(\theta_{Tr}, x_T, y_T) \in H_{Tr} \land (d, \theta_{Tr}) \in A_{Tr}]
    \end{gather*}
    where
    $$x_T \in X_T, y_T \in Y_T, d \in D, \theta_{Tr} \in \Theta_{Tr}.$$
    The nature of source knowledge $K_S$\footnote{Here transferred knowledge $K_S$ is defined as $D_S$ and $\Theta_S$, the source data and parameters, following convention \cite{pan2009survey}.}, the transfer learning algorithm $A_{Tr}$, hypotheses $H_{Tr}$, and parameters $\Theta_{Tr}$ specify transfer learning as a relation on $\overline{S_S}$ and $\overline{S_T}$. 
    \label{def:tl}
\end{definition}
This AST model of transfer learning is depicted in Figure \ref{fig:transfer-learning-systems}. Previous work extensively elaborates on and beyond Definition \ref{def:tl} \cite{cody2021systems}.

\begin{figure*}[t]
    \centering
    \includegraphics[width=.8\textwidth]{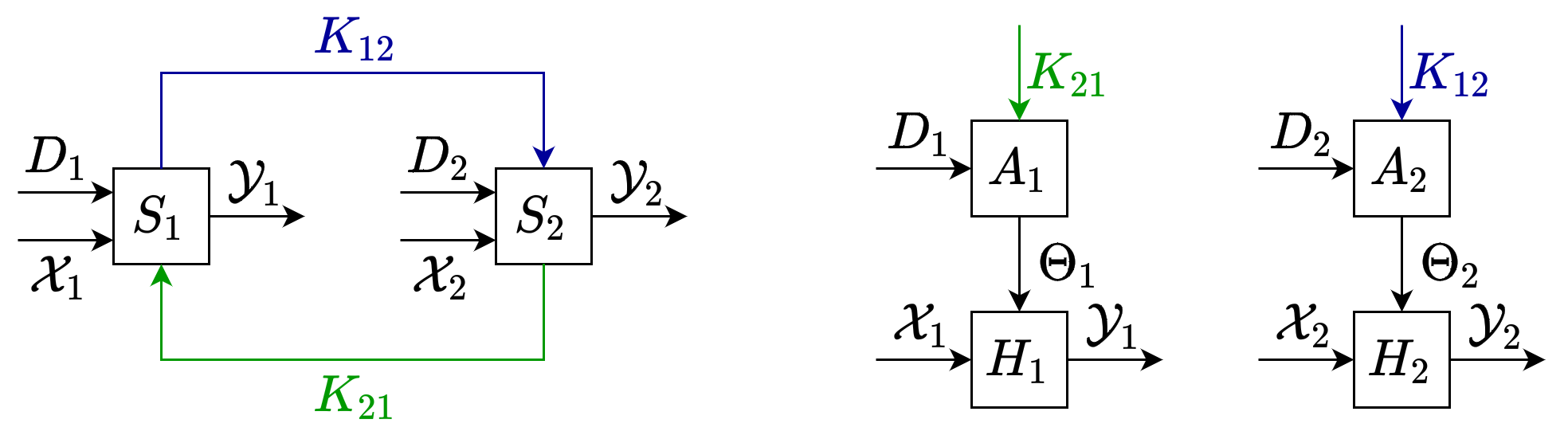}
    \caption{Two learning systems transferring knowledge to each other depicted at the elementary (left) and cascade (middle, right) levels of abstraction \cite{cody2021systems}.}
    \label{fig:transfer-learning-systems}
\end{figure*}

\section{Multi-Task and Meta-Learning Systems}
\label{sec:mtlml}

\subsection{Multi-Task Learning}

Zhang and Yang define multi-task learning as \cite{zhang2021survey},
\begin{quote}
    ``a learning paradigm in machine learning and its aim is to leverage useful information contained in multiple related tasks to help improve the generalization performance of all the tasks.'' 
\end{quote}
Multi-task learning systems are defined herein as follows.
%\begin{definition}[Multi-Task Learning System.] \\
%    Multi-task learning systems are learning systems $S \subset \times \{A, \Theta, D, H, \mathcal{X}, \mathcal{Y}\}$ with hypotheses $H$ that specifically consists of $N$ function-spaces $H = (H_1, ..., H_N)$ where $N >= 2$, and with parameters $\Theta$, inputs $\mathcal{X}$, and outputs $\mathcal{Y}$ indexed by $n \in N$, i.e.,
%    \begin{gather*}
%        \Theta = (\Theta_1, ... ,\Theta_N), \\
%        \mathcal{X} = (\mathcal{X}_1, ..., \mathcal{X}_N), \\
%        \mathcal{Y} = (\mathcal{Y}_1, ..., \mathcal{Y}_N.
%    \end{gather*}
%    such that $A: D \to (\Theta_1, ..., \Theta_N)$.
%\end{definition}
\begin{definition}[Multi-Task Learning Systems.] \\
    Given $N$ learning systems $S_1, ..., S_N$, a multi-task learning system is a learning system $S \subset \times \{A, D, \Theta, H, \mathcal{X}, \mathcal{Y}\}$ where,
    \begin{gather*}
        D = (D_1, ..., D_N), 
        H = (H_1, ..., H_N), \\
        \Theta = (\Theta_1, ..., \Theta_N), 
        \mathcal{X} = (\mathcal{X}_1, ..., \mathcal{X}_N), \\
        \mathcal{Y} = (\mathcal{Y}_1, ..., \mathcal{Y}_N),
    \end{gather*}
    i.e., $A: (D_1, ..., D_N) \to (\Theta_1, ..., \Theta_N)$.
\end{definition}
Multi-task learning systems are simply learning systems that jointly learn multiple, distinct hypotheses. Multi-task learning systems are depicted in Figure \ref{fig:multi-and-meta-learning-systems}A. A trivial multi-task learning system can be defined as follows.
\begin{definition}[Trivial Multi-Task Learning Systems.] \\
    Given $N$ learning systems $S_1, ..., S_N$, a trivial multi-task learning system is a multi-task learning system $S \subset \times \{A, D, H, \Theta, \mathcal{X}, \mathcal{Y}\}$ defined over $S_1, ..., S_N$ where $A = (A_1, ..., A_N)$. 
\end{definition}
In other words, the trivial case of multi-task learning is a superficial grouping of algorithms $(A_1, ..., A_N)$ where $A$ simply uses $D_1, ..., D_N$ as input to each respective algorithm $A_n$ for $n \in N$. A non-trivial multi-task learning system can be defined as follows.
\begin{definition}[Non-Trivial Multi-Task Learning Systems.] \\
    Given $N$ learning systems $S_1, ..., S_N$, a non-trivial multi-task learning system is a multi-task learning system $S \subset \times \{A, D, H, \Theta, \mathcal{X}, \mathcal{Y}\}$ defined over $S_1, ..., S_N$ where $A \neq (A_1, ..., A_N)$.
\end{definition}
\begin{figure*}[t]
    \centering
    \includegraphics[width=.6\textwidth]{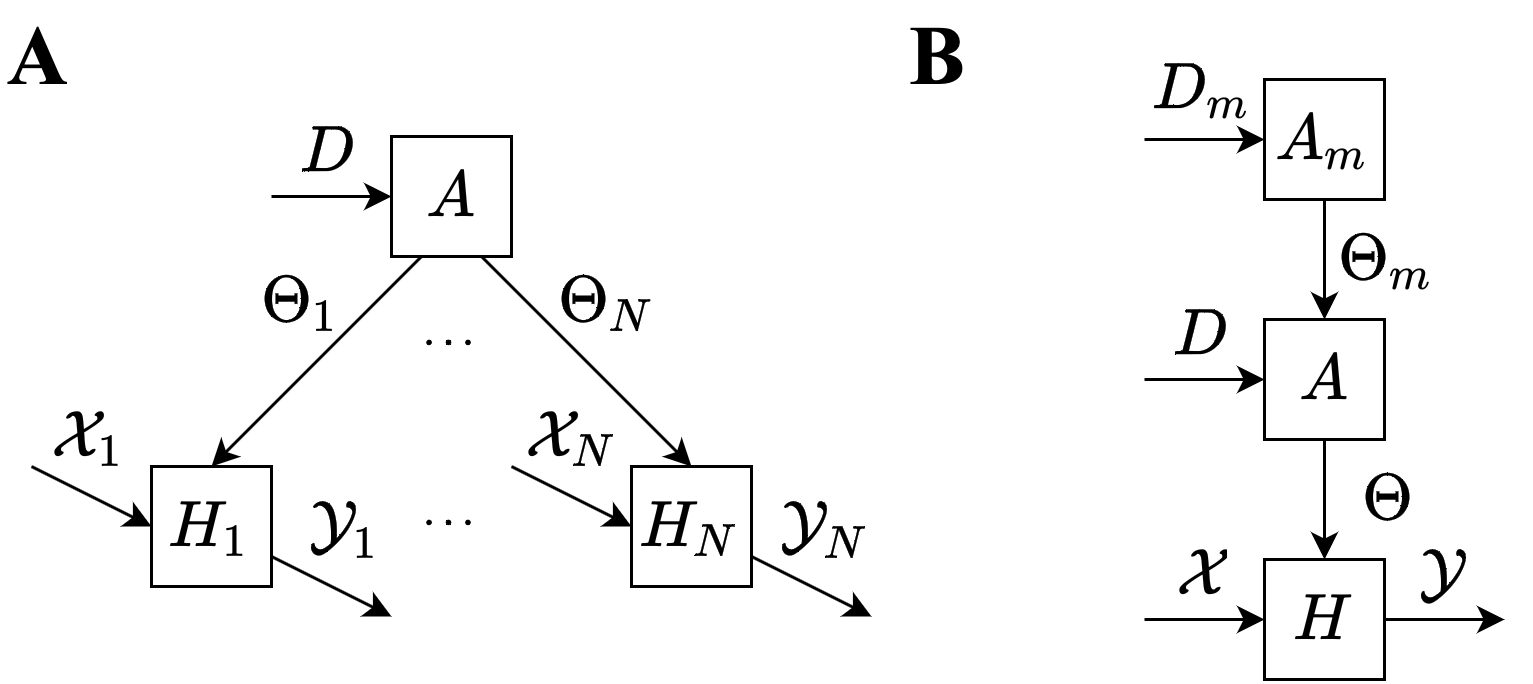}
    \caption{A multi-task learning system (A) and a meta-learning system (B).}
    \label{fig:multi-and-meta-learning-systems}
\end{figure*}

\subsection{Meta-Learning}

Vanschoren defines meta-learning as \cite{vanschoren2018meta}
\begin{quote}
    ``the science of systematically observing how different machine learning approaches perform on a wide range of learning tasks, and then learning from this experience, or meta-data, to learn new tasks much faster than otherwise possible.''
\end{quote}
Meta-learning systems are defined herein as follows.
\begin{definition}[Meta-Learning System.] \\
    Meta-learning systems are learning systems $S \subset \times \{A_m, \Theta_m, D_m, H_m, \mathcal{X}_m, \mathcal{Y}_m\}$ with hypotheses $H_m$ that are algorithms $A$, inputs $\mathcal{X}_m $ that are data $D$, outputs $\mathcal{Y}_m$ that are parameters $\Theta$ for hypotheses $H:\Theta \times \mathcal{X} \to \mathcal{Y}$, and where $S \subset \times \{A, D, \Theta, H, \mathcal{X}, \mathcal{Y}\}$ is a learning system.
\end{definition}
Meta-learning systems are learning systems whose hypotheses are learning algorithms. Meta-learning systems are depicted in Figure \ref{fig:multi-and-meta-learning-systems}B.

\section{Homomorphisms Between Learning Systems}
\label{sec:comp}

% informally and formally define homomorphism

Similarity of systems is a fundamental notion. Structural similarity describes the \emph{homomorphism} between two systems' structures. In accord with category theory, a map from one system to another is termed a morphism. Homomorphism specifies the morphism to be onto. Homomorphism is formally defined as follows.
\begin{definition}{\emph{Homomorphism}.} \\
    An input-output system $S_1 \subset \times \{\mathcal{X}_1 \times \mathcal{Y}_1\}$ is homomorphic to $S_2 \subset \times \{\mathcal{X}_2, \mathcal{Y}_2\}$ if there exists a pair of maps,
    \begin{align*}
        \varrho:\mathcal{X}_1 \to \mathcal{X}_2, \vartheta:\mathcal{Y}_1 \to \mathcal{Y}_2
    \end{align*}
    such that for all $x_1\in \mathcal{X}_1$, $x_2\in \mathcal{X}_2$, and $y_1\in \mathcal{Y}_1$, $y_2\in \mathcal{Y}_2$, $\varrho(x_1)=x_2$ and $\vartheta(y_1)=y_2$.
\end{definition}

% Prove the homomorphism between transfer and multi-task learning systems
Let a two-way transfer learning system be a pair of transfer learning systems that both transfer knowledge to each other. In the following, it is proven that two transfer learning systems sharing knowledge with each other are homomorphic to a non-trivial multi-task learning system, as depicted in Figure \ref{fig:transfer-to-multi}.

\begin{theorem}
    Two-way transfer learning systems are homomorphic to a non-trivial multi-task learning system.
\end{theorem}
\begin{proof} %[Two-way transfer learning systems are homomorphic to a non-trivial multi-task learning system]\\
    Consider two learning systems $S_1'$ and $S_2'$,
    \begin{gather*}
        S_1' \subset \times \{A_1', D_1', \Theta_1', H_1', \mathcal{X}_1, \mathcal{Y}_1\}, \\
        S_2' \subset \times \{A_2', D_2', \Theta_2', H_2', \mathcal{X}_2, \mathcal{Y}_2\}.
    \end{gather*} Let transfer learning be used to transfer knowledge $K_{12} \subset D_1'$ from $S_1'$ to $S_2'$ and knowledge $K_{21} \subset D_2'$ from $S_2'$ to $S_1'$. This creates two transfer learning systems, termed $S_1$ and $S_2$, respectively, 
    \begin{gather*}
        S_1 \subset \times \{A_1, D_1, \Theta_1, H_1, \mathcal{X}_1, \mathcal{Y}_1\}, \\
        S_2 \subset \times \{A_2, D_2, \Theta_2, H_2, \mathcal{X}_2, \mathcal{Y}_2\}. \\
    \end{gather*} where $D_1 \subset D_1' \times D_2'$ and $D_2 \subset D_1' \times D_2'$, as in Figure \ref{fig:transfer-learning-systems}. Consider a multi-task learning system 
    \begin{gather*}
        S \subset \times \{A, D, \Theta, H, \mathcal{X}, \mathcal{Y}\}
    \end{gather*} 
    such that $A = (A_1, A_2)$, $D = (D_1, D_2)$, $\Theta = (\Theta_1, \Theta_2)$, $H = (H_1, H_2)$, $\mathcal{X} = (\mathcal{X}_1, \mathcal{X}_2)$, and $\mathcal{Y} = (\mathcal{Y}_1, \mathcal{Y}_2)$. Clearly, by the identity, $A$, $D$, $\Theta$, $H$, $\mathcal{X}$, and $\mathcal{Y}$ are homomorphic to $(A_1, A_2)$, $(D_1, D_2)$, $(\Theta_1, \Theta_2)$, $(H_1, H_2)$, $(\mathcal{X}_1, \mathcal{X}_2)$, and $(\mathcal{Y}_1, \mathcal{Y}_2)$, respectively. %, and, similarly, the latter are homomorphic to the former (i.e., they are isomorphic). 
    %Therefore we can define a set of onto maps $\varrho_D$, $\varrho_{\Theta}$, $\varrho_H$, $\varrho_{\mathcal{X}}$, and $\varrho_{\mathcal{Y}}$ from the appropriate component sets of $(S_1, S_2)$ to those of $S$. Now, let $\varrho_A$ be a map $\varrho_A: A_1 \times A_2 \to A$. Since the support of $A$ is $(D_1, D_2)$, and that of $A_1$ and $A_2$ are $D_1$ and $D_2$, respectively, it holds that there exists an onto map for $\varrho_A$. To see this, note the definition of learning systems gives us
    %\begin{gather*}
    %    (d_n, x_n, y_n) \in \mathcal{P}(S_n) \leftrightarrow \\
    %    (\exists \theta_n)[(\theta_n, x_n, y_n) \in H_n \land (d_n, \theta_n) \in A_n] \\
    %    : d_n \in D_n, \theta_n \in \Theta_n, x_n \in \mathcal{X}_n, y_n \in \mathcal{Y}_n
    %\end{gather*}
    %for a learning system $S_n$. Thus, by definition, 
    %\begin{gather*}
    %    (d_1, k_{21}, \theta_1) \in A_1 \land (d_2, k_{12}, \theta_2) \in A_2 \leftrightarrow \\
    %    ((d_1, d_2),(\theta_1, \theta_2)) \in A \\
    %    : d_1 \in D_1, k_{21} \in K_{21}, \theta_1 \in \Theta_1, d_2 \in D_2, k_{12} \in K_{12}, \theta_2 \in \Theta_2
    %\end{gather*}
    Thus, there exists a set of onto maps $\{\varrho_A, \varrho_D, \varrho_{\Theta}, \varrho_{H}, \varrho_{\mathcal{X}}, \varrho_{\mathcal{Y}}\}$ from $(S_1, S_2) \to S$. Thus, the two-way transfer learning system $(S_1, S_2)$ is homomorphic to the multi-task learning system $S$. Since $A = (A_1, A_2)$ and since $A_1' \neq A_1$ and $A_2' \neq A_2$ (they have different supports), $S$ is therefore necessarily a non-trivial multi-task learning system with respect to $S_1'$ and $S_2'$. $\square$
\end{proof}

\begin{figure*}[t]
    \centering
    \includegraphics[width=.9\textwidth]{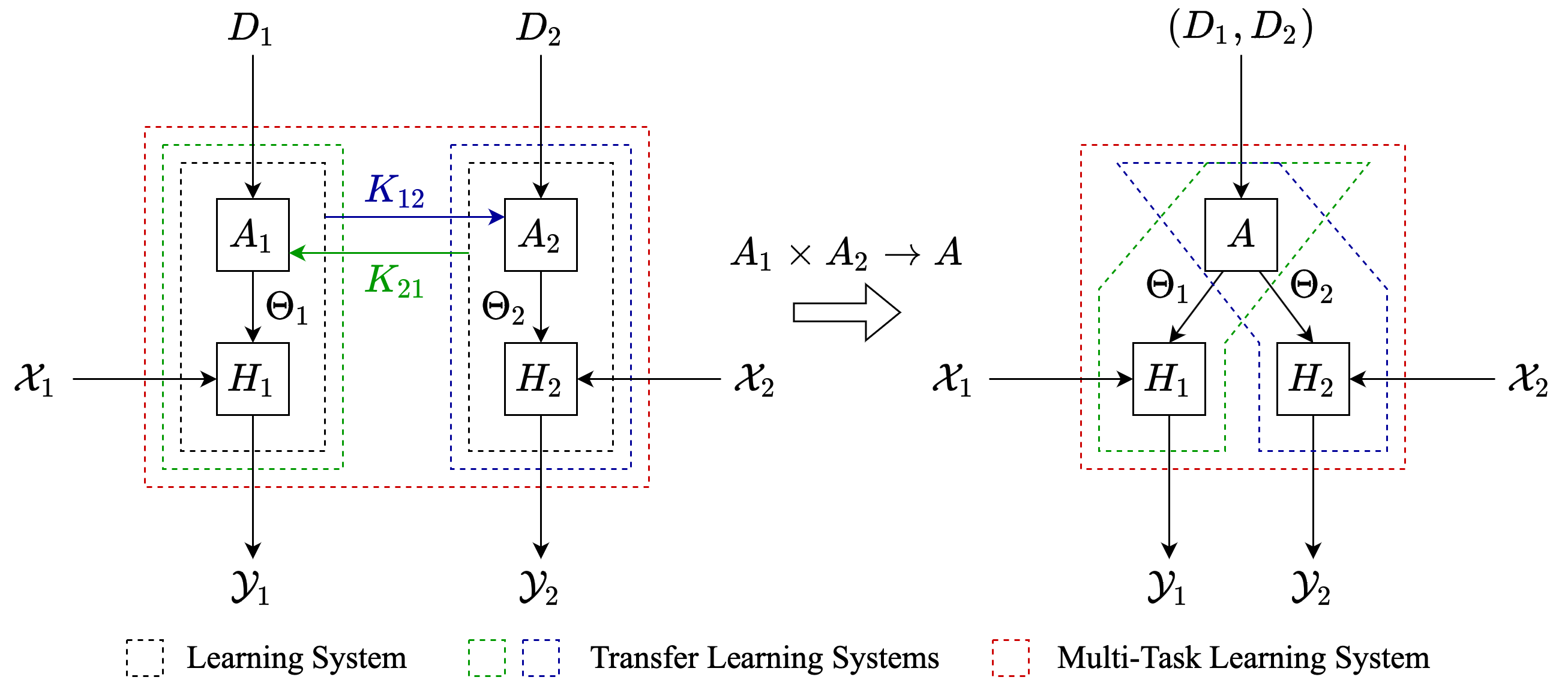}
    \caption{A set of learning systems all transferring knowledge to each other implicitly forms a multi-task learning system.}
    \label{fig:transfer-to-multi}
\end{figure*}

The above hints that multi-task learning systems are related to parallel connections of learning systems. To see this, first define a parallel connection as follows.
\begin{definition} [Parallel Connections] \\
    A parallel connection of systems $S_1: \mathcal{X}_1 \to \mathcal{Y}_1$ and $S_2: \mathcal{X}_2 \to \mathcal{Y}_2$ is an operator $\Vert: \overline{S_1} \times \overline{S_2} \to \overline{S_2}$ such that $S_2: (\mathcal{X}_1 \times \mathcal{X}_2) \to (\mathcal{Y}_1 \times \mathcal{Y}_2)$ and 
    \begin{gather*}
        ((x_1, x_2), (y_1, y_2)) \in S_2 \leftrightarrow \\
        ((x_1, y_1)) \in S_1 \wedge ((x_2, y_2) \in S_2).
    \end{gather*}
\end{definition}
\begin{theorem}
    Trivial multi-task systems are a parallel connection of learning systems.
\end{theorem}
\begin{proof}
    Consider a set of $N$ learning systems $S_1$, ..., $S_N$. Let $S_{1-2} = S_1 \Vert S_2$. Let $S_{1-n} = S_{1-(n-1)} \Vert S_n$. Thus, $S_{1-N}$, at the elementary (input-output) level of abstraction is,
    \begin{gather*}
        S_{1-N} = (D_1 \times ... \times D_N) \times (\mathcal{X}_1 \times ... \times \mathcal{X}_N) \to (\mathcal{Y}_1 \times ... \times \mathcal{Y}_N),
    \end{gather*}
    which simplifies to,
    \begin{gather*}
        S_{1-N}: (D_1, ..., D_N) \times (\mathcal{X}_1, ..., \mathcal{X}_N) \to (\mathcal{Y}_1, ..., \mathcal{Y}_N).
    \end{gather*}
    Let $D = (D_1, ..., D_N)$, $\mathcal{X} = (\mathcal{X}_1, ..., \mathcal{X}_N)$, and $\mathcal{Y} = (\mathcal{Y}_1, ..., \mathcal{Y}_N)$. Let $A$, $\Theta$, and $H$ be defined similarly. By definition, the system 
    \begin{gather*}
    S \subset \times \{A, D, \Theta, H, \mathcal{X}, \mathcal{Y}\}
    \end{gather*}
    is a trivial multi-task learning system. $\square$
\end{proof}

When triviality does not hold, multi-task learning is a \emph{shallow} parallel connector in the sense that it is always a parallel connector in (elementary-level) terms of $D$, $\mathcal{X}$, and $\mathcal{Y}$, but not always a parallel connector in the (cascade-level) terms of the relation on them given by $A$, $\Theta$, and $H$ since $\exists A \neq (A_1, ..., A_N)$. 

In contrast to the parallel connections of multi-task learning, meta-learning is related to cascade connections and hierarchy. To see this, first define a cascade connection as follows.
\begin{definition} [Cascade Connections] \\
    Let $\circ: \overline{S} \times \overline{S} \to \overline{S}$ be such that $S_1 \circ S_2 = S_3$, where,
    \begin{gather*}
        S_1 \subset X_1 \times (Y_1 \times (Z_1)), S_2 \subset (X_2 \times Z_2) \times Y_2 \\
        S_3 \subset (X_1 \times X_2) \times (Y_1 \times Y_2), Z_1 = Z_2 = Z
    \end{gather*}
    and,
    \begin{gather*}
        ((x_1, x_2), (y_1, y_2)) \in S_3 \leftrightarrow \\
        (\exists z) ((x_1, (y_1, z)) \in S_1 \wedge ((x_2, z), y_2) \in S_2)
    \end{gather*}
    $\circ$ is termed the cascade (connecting) operator.
\end{definition}
\begin{theorem}
    Meta-learning systems are a cascade connection of a learning algorithm and a learning system.
\end{theorem}
\begin{proof}
    Consider a learning algorithm $A_2: D_2 \to \Theta_2$ and a learning system $S_1$. Let $S_3 = A_2 \circ S_1$. Thus $S_3 = \subset \times \{A_2, D_2, \Theta_2, A_1, D_1, \Theta_1, H_1, \mathcal{X}_1, \mathcal{Y}_1\}$ where $A_1: \Theta_2 \times D_1 \to \Theta_1$. Let $S_m \subset S_3$ such that $S_m \subset \times \{A', D', \Theta', A, D, \Theta\}$. $S_m$ is a meta-learning system where $A_1$, $D_1$, and $\Theta_1$ are the hypotheses, inputs, and outputs of $S_m$. $\square$
\end{proof}

%%% !!! Note in conclusion that these are basic structural statements, and of course multi-task learning may benefit from not having Theta, etc., be isomorphic, i.e., in changing other aspects of the system of transfer learning systems than the algorithms.
% looking for about 1-2 pages of content on compositions not including figures. Maybe do a proof that a fully connected N system transfer learning network is a N-task multi-task learning system. Note that this relies on parallel connections. Then talk about meta-learning as a hierarchy because of its series connection. Then talk about how these are basic parallel and series compositions for learning systems (expressed at the cascade level), and should be viewed as fundamental building blocks in learning system engineering. Also mention how this relates to choice of representation by modeler: because there are homomorphisms modelers can move between representations while preserving structure (i.e., onto mappings between them). We depict this with the big-learning-system-fig by going from two meta-transfer learning systems to a generic learning system.

\subsection{Discussion}
\label{sec:disc}

The preceding provides the basic math needed to support the use of transfer, multi-task, and meta-learning as basic elements of modeling abstract learning systems. Because of the generality and homomorphism of these three concepts, a systems modeler clearly has many representations to choose from when modeling learning systems. Consider the learning system shown in Figure \ref{fig:big-learning-systems-fig}, which shows a series of homomorphisms on a learning system from Figure \ref{fig:big-learning-systems-fig}A to \ref{fig:big-learning-systems-fig}F.

Figure \ref{fig:big-learning-systems-fig}A shows two learning systems with meta-learning systems that transfer knowledge to each other. In Figure \ref{fig:big-learning-systems-fig}B, a parallel connection of $A_1$ and $A_2$ transforms the system in $\ref{fig:big-learning-systems-fig}$A into a multi-task learning system with a decomposed meta-learning system. In Figure \ref{fig:big-learning-systems-fig}C, the meta-learning system is recomposed into $A_m$ using a parallel connection. Figure \ref{fig:big-learning-systems-fig}D shows the hypotheses of the multi-task learning system are composed into $H$ by a parallel connection. Figure \ref{fig:big-learning-systems-fig}E shows the meta-learning hierarchy is collapsed into algorithm $A'$ using a series connection. And lastly, in Figure \ref{fig:big-learning-systems-fig}F, sets are redefined to recover a general learning system as in Definition \ref{def:ls}. All of these morphisms are homomorphisms---they are onto and thus structure-preserving.

\begin{figure*}[!tp]
    \centering
    \includegraphics[width=.99\textwidth]{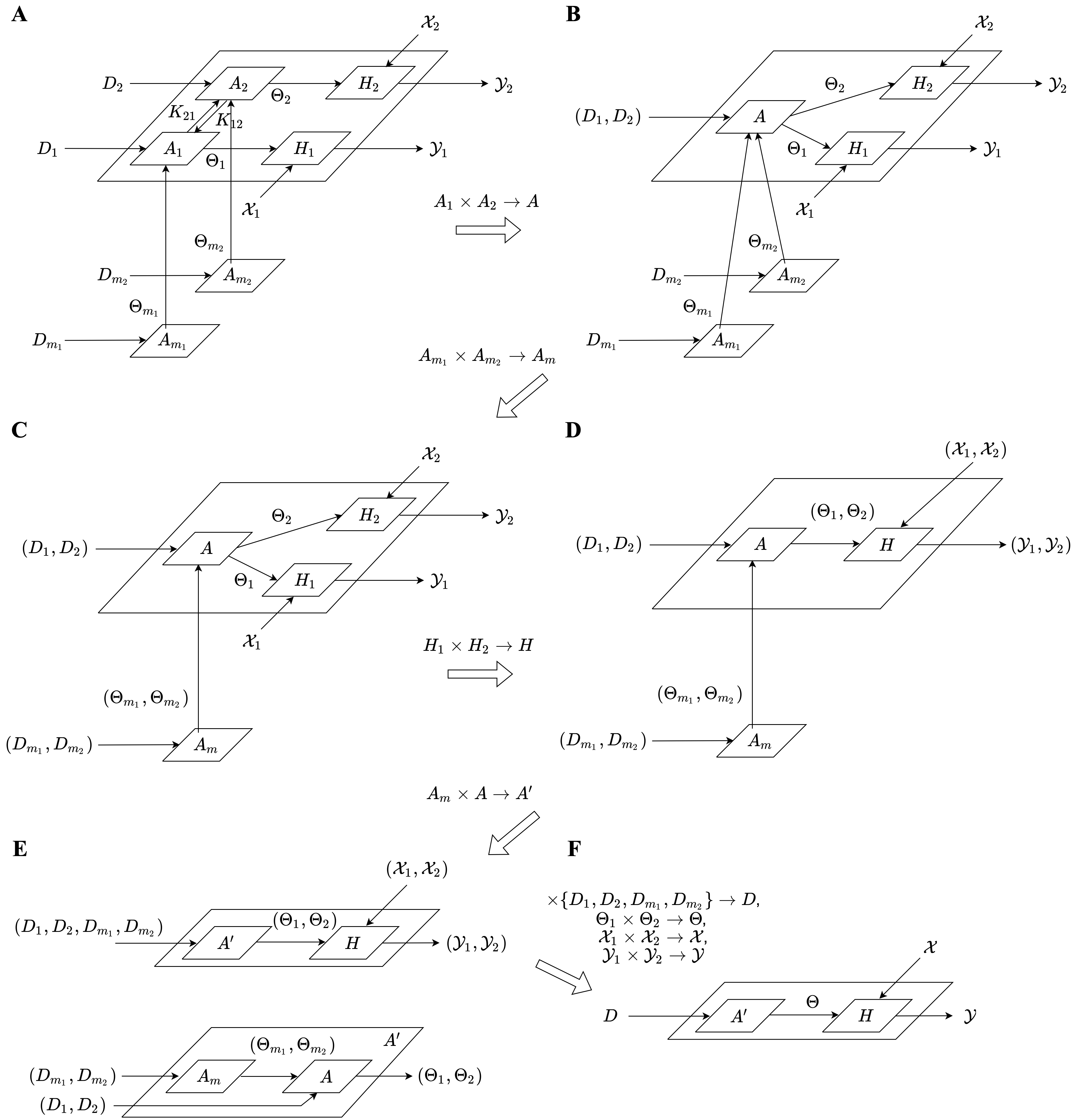}
    \caption{A series of homomorphisms from a two-way transfer learning system with meta-learning in (A) to a generic learning system in (F), as described in detail in Section 
    \ref{sec:disc}.}
    \label{fig:big-learning-systems-fig}
\end{figure*}

\section{Conclusion}

%At the elementary level, learning systems are composable in the usual ways, serial connections, parallel connections, etc. However, at the cascade level, one can define compositions that inter-link learning systems at a deeper level of abstraction. In particular, the much discussed and often jointly discussed concepts of transfer, multi-task, and meta-learning can be formally (homomorphically) related in terms of compositions at the cascade level of abstraction.

Appreciating the compositional and hierarchical relations between the three concepts of learning makes clear a simple point about learning systems. Transfer, multi-task, and meta-learning are basic compositions of learning systems that display a recursive self-similarity. A multi-task learning system can use transfer learning from other multi-task learning systems that all have meta-learning systems that use transfer learning, and so on. 

What is a multi-task transfer learning system? What is a transfer multi-task learning system? What is a meta-transfer multi-task learning system? What is a meta-transfer multi-task transfer learning system? What is ... etc.? Clearly what is a useful taxonomy for organizing machine learning solution methods in the literature becomes burdensome and tedious when applied to systems modeling. Using the presented material, modelers have simple, general formalism for describing the compositions of learning systems possible by way of transfer, multi-task, and meta-learning. 

\bibliographystyle{splncs04}
\bibliography{ref}

\end{document}